\documentclass{article}
\usepackage[utf8]{inputenc}
\usepackage{fullpage}
\usepackage{natbib}
\usepackage{xcolor,xspace}
\usepackage{dsfont}
\usepackage[normalem]{ulem}

\usepackage{algpseudocode}
\usepackage{algcompatible}
\usepackage{algorithm}

\usepackage[T1]{fontenc}    
\usepackage{hyperref}       
\usepackage{url}            
\usepackage{booktabs}       
\usepackage{amsfonts}       
\usepackage{nicefrac}       
\usepackage{microtype}      
\usepackage{xcolor, xspace}         
\usepackage{amsmath,amssymb,mathtools}
\usepackage{algpseudocode}
\usepackage{algcompatible}

\usepackage{dsfont}
\usepackage{times}

\usepackage{amsmath,amsthm,amssymb,mathtools,thmtools}

\newcommand{\U}{\mathcal{U}}
\newcommand{\V}{\mathcal{V}}
\newcommand{\W}{\mathcal{W}}
\newtheorem{theorem}{Theorem}[section]
\newtheorem{lemma}[theorem]{Lemma}

\newtheorem{corollary}[theorem]{Corollary}

\newtheorem{remark}[theorem]{Remark}

\newtheorem{definition}[theorem]{Definition}

\newtheorem{property}{Property}
   \newcommand{\reals}{\mathbb{R}}
   \newcommand{\naturals}{\mathbb{N}}
   \newcommand{\Ex}{\mathbb{E}}
   \renewcommand{\Pr}{\mathbb{P}}
   \newcommand{\Lo}[1]{{\mathcal L_{#1}}}
   \newcommand{\bLo}[1]{{\mathcal L^{0/1}_{#1}}}
   \newcommand{\rLo}[2]{{\mathcal L^{#1}_{#2}}}
   \newcommand{\lo}{\ell}
   \newcommand{\blo}{\ell^{0/1}}
   \newcommand{\rlo}[1]{\ell^{#1}}

    \newcommand{\indct}[1]{\mathds{1}\left[{#1}\right]}
   \newcommand{\B}{{\mathcal B}}
   \renewcommand{\P}{{\mathcal P}}
   \newcommand{\A}{{\mathcal A}}
   \newcommand{\Ocal}{{\mathcal O}}

  \renewcommand{\H}{{\mathcal H}}

  \newcommand{\Acal}{{\mathcal A}}

  \newcommand{\Hcal}{{\mathcal H}}

  \newcommand{\Mcal}{{\mathcal M}}
  
  \newcommand{\vc}{\mathrm{VC}}

  \newcommand{\e}{\mathrm{e}}
  \renewcommand{\d}{\mathrm{dist}}

  \newcommand{\iid}{i.i.d.~}

\newcommand{\cH}{\mathcal{H}}

\newcommand{\cO}{\mathcal{O}}

\newcommand{\tpas}{TPaS\xspace}

\newcommand{\modify}[1]{#1}

\title{Adversarially Robust Learning with Tolerance}

\author{
Hassan Ashtiani \thanks{McMaster University, \texttt{zokaeiam@mcmaster.ca}. Hassan Ashtiani is also a faculty affiliate at Toronto's Vector Institute and was supported by an NSERC Discovery Grant.} \and
Vinayak Pathak \thanks{Layer 6 AI, \texttt{vinayak@layer6.ai}.} \and
Ruth Urner \thanks{York University, \texttt{ruth@eecs.yorku.ca}. Ruth Urner is also a faculty affiliate at Toronto's Vector Institute and was supported by an NSERC Discovery Grant.}
}

\begin{document}

\maketitle

    \begin{abstract}
We initiate the study of tolerant adversarial PAC-learning with respect to metric perturbation sets. In adversarial PAC-learning, an adversary is allowed to replace a test point $x$ with an arbitrary point in a closed ball of radius $r$ centered at $x$. In the tolerant version, the error of the learner is compared with the best achievable error with respect to a slightly larger perturbation radius $(1+\gamma)r$. This simple tweak helps us bridge the gap between theory and practice and obtain the first PAC-type guarantees for algorithmic techniques that are popular in practice.

Our first result concerns the widely-used ``perturb-and-smooth'' approach for adversarial learning. For perturbation sets with doubling dimension $d$, we show that a variant of these approaches PAC-learns any hypothesis class $\mathcal{H}$ with VC-dimension $v$ in the $\gamma$-tolerant adversarial setting with $O\left(\frac{v(1+1/\gamma)^{O(d)}}{\varepsilon}\right)$ samples. This is in contrast to the traditional (non-tolerant) setting in which, as we show, the perturb-and-smooth approach can provably fail.

Our second result shows that one can PAC-learn the same class using 
$\widetilde{O}\left(\frac{d.v\log(1+1/\gamma)}{\varepsilon^2}\right)$ samples even in the agnostic setting. This result is based on a novel compression-based algorithm, and achieves a linear dependence on the doubling dimension as well as the VC-dimension. This is in contrast to the non-tolerant setting where there is no known sample complexity upper bound that depends polynomially on the VC-dimension.
\end{abstract}

\section{Introduction}

Several empirical studies~\citep{SzegedyZSBEGF13, GoodfellowMP18} have demonstrated that models trained to have a low accuracy on a data set often have the undesirable property that a small perturbation to an input instance can change the label outputted by the model. For most domains this does not align with human perception and thus indicates that the learned models are not representing the ground truth despite obtaining good accuracy on test sets.

The theory of PAC-learning characterizes the conditions under which learning is possible. For binary classification, the following conditions are sufficient: a) unseen data should arrive from the same distribution as training data, and b) the class of models should have a low capacity (as measured, for example, by its VC-dimension). If these conditions are met, an \emph{Empirical Risk Minimizer} (ERM) that simply optimizes model parameters to maximize accuracy on the training set learns successfully. 

Recent work has studied test-time adversarial perturbations under the PAC-learning framework. If an adversary is allowed to perturb data during test time then the conditions above do not hold, and we cannot hope for the model to learn to be robust just by running ERM. Thus, the goal here is to bias the learning process towards finding models where label-changing perturbations are rare. This is achieved by defining a loss function that combines both classification error and the probability of seeing label-changing perturbations, and learning models that minimize this loss on unseen data. It has been shown that even though (robust) ERM can fail in this setting, PAC-learning is still possible as long as we know during training the kinds of perturbations we want to guard against at test time~\citep{MontasserHS19}. This result holds for all perturbation sets.
However, the learning algorithm is significantly more complex than robust ERM and requires a large number of samples (with the best known sample complexity bounds potentially being exponential in the VC-dimension).

We study a \emph{tolerant} version of the adversarially robust learning framework and restrict the perturbations to balls in a general metric space with finite doubling dimension.
We show this slight shift in the learning objective yields significantly improved sample complexity bounds through a simpler learning paradigm than what was previously known. In fact, we show that a version of the common ``perturb-and-smooth'' paradigm successfully PAC-learns any class of bounded VC-dimension in this setting.

{\bf Learning in general metric spaces.}
What kinds of perturbations should a learning algorithm guard against? Any transformation of the input that we believe should not change its label could be a viable perturbation for the adversary to use. The early works in this area considered perturbations contained within a small $\ell_p$-ball of the input. More recent work has considered other transformations such as a small rotation, or translation of an input image~\citep{engstrom2019exploring, fawzi2015manitest, kanbak2018geometric, xiao2018spatially}, or even adding small amounts of fog or snow~\citep{kang2019testing}. It has also been argued that small perturbations in some \emph{feature space} should be allowed as opposed to the input space~\citep{inkawhich2019feature, sabour2016adversarial, xu2020towards, song2018constructing, hosseini2018semantic}. This motivates the study of more general perturbations.

We consider a setting where the input comes from a domain that is equipped with a distance metric and allows perturbations to be within a small metric ball around the input. Earlier work on general perturbation sets (for example,~\citep{MontasserHS19}) considered arbitrary perturbations. In this setting one does not quantify the magnitude of a perturbation and thus cannot talk about small versus large perturbations. Modeling perturbations using a metric space enables us to do that while also keeping the setup general enough to be able to encode a large variety of perturbation sets by choosing appropriate distance functions.

{\bf Learning with tolerance.}
In practice, we often believe that small perturbations of the input should not change its label but we do not know \emph{precisely} what small means. However, in the PAC-learning framework for adversarially robust classification, we are required to define a precise perturbation set and learn a model that has error arbitrarily close to the smallest error that can be achieved with respect to that perturbation set. In other words, we aim
to be arbitrarily close to a target that was picked somewhat arbitrarily to begin with. Due to the uncertainty about the correct perturbation size, it is more meaningful to allow for a wider range of error values. To achieve this, we introduce the concept of tolerance. In the tolerant setting, in addition to specifying a perturbation size $r$, we introduce a tolerance parameter $\gamma$ that encodes our uncertainty about the size of allowed perturbations. Then, for any given $\varepsilon>0$, we aim 
to learn a model whose error with respect to perturbations of size $r$ is at most  $\varepsilon$ more than the smallest error achievable with respect to perturbations of size $r(1+\gamma)$. 

\section{Our results}
In this paper we formalize and initiate the study of the problem of adversarially robust learning in the tolerant setting for general metric spaces and provide two algorithms for the task. Both of our algorithms rely on: 1) modifying the training data by randomly sampling points from the perturbation sets around each data point, and 2) smoothing the output of the model by taking a majority over the labels returned by the model for nearby points.

Our first algorithm starts by modifying the training set by randomly perturbing each training point using a certain distribution (see Section~\ref{sec:erm} for details). It then trains a (non-robust) PAC-learner (such as ERM) on the perturbed training set to find a hypothesis $h$. Finally, it outputs a smooth version of $h$. The smoothing step replaces $h(x)$ at each point $x$ with the a majority label outputted by $h$ on the points around $x$. We show that for metric spaces of a fixed doubling dimension, this algorithm successfully learns in the 
tolerant setting~\modify{assuming tolerant realizability}.

\begin{theorem}[Informal version of Theorem~\ref{thm:tpas_guarantee}]
Let $(X, \d)$ be a metric space with doubling dimension $d$ and $\cH$ a hypothesis class. Assuming~\modify{tolerant} realizability, $\cH$ can be learned tolerantly in the adversarially robust setting using $O\left(\frac{(1+1/\gamma)^{O(d)}\vc(\cH)}{\varepsilon}\right)$ samples, where $\gamma$ encodes the amount of allowed tolerance, and $\varepsilon$ is the desired accuracy.
\end{theorem}
An interesting feature of the above result is the linear dependence of the sample complexity with respect to $\vc(\cH)$. This is in contrast to the best known upper bound for non-tolerant adversarial setting ~\citep{MontasserHS19} which depends on the \emph{dual VC-dimension} of the hypothesis class and in general is exponential in $\vc(\cH)$. 
 Moreover, this is the first PAC-type guarantee for the general perturb-and-smooth paradigm, indicating that the tolerant adversarial learning is the ``right'' learning model for studying these approaches.
While the above method enjoys simplicity and can be computationally efficient, one downside is that its sample complexity grows exponentially with the doubling dimension. For instance, such algorithm cannot be used on high-dimensional data in the Euclidean space. Another limitation is that the guarantee holds only in the (robustly) realizable setting.

In the second main part of our submission (Section \ref{sec:agn}) we show that, surprisingly, these limitations can be overcome by incorporating ideas from the tolerant framework and perturb-and-smooth algorithms into a novel compression scheme for robust learning. The resulting algorithm improves the dependence on the doubling dimension, and works in the general agnostic setting.
\begin{theorem}[Informal version of Corollary~\ref{Thm:compress_metric}]
Let $(X, \d)$ be a metric space with doubling dimension $d$ and $\cH$ a hypothesis class. 
Then $\cH$ can be learned tolerantly in the adversarially robust setting using $\widetilde{O}\left(\frac{O(d)\vc(\cH)\log(1+1/\gamma)}{\varepsilon^2}\right)$ samples, where $\widetilde{O}$ hides logarithmic factors, $\gamma$ encodes the amount of allowed tolerance, and $\varepsilon$ is the desired accuracy.
\end{theorem}
This algorithm exploits the connection between sample compression and adversarially robust learning~\cite{MontasserHS19}. However, unlike~\cite{MontasserHS19}, our new compression scheme sidesteps the dependence on the dual VC-dimension (refer to the discussion at the end of Section~\ref{sec:agn} for more details). As a result, we get an exponential improvement over the best known (nontolerant) sample complexity in terms of dependence on VC-dimension.

\section{Related work}
PAC-learning for adversarially robust classification has been studied extensively in recent years~\citep{cullina2018pac, awasthi2019robustness, MontasserHS19, feige2015learning, attias2018improved, montasser2020efficiently, ashtiani2020black}. These works provide learning algorithms that guarantee low generalization error in the presence of adversarial perturbations in various settings.
 The most general result is due to~\cite{MontasserHS19}, and is proved for general hypothesis classes and perturbation sets. All of the above results assume that the learner knows the kinds of perturbations allowed for the adversary. Some more recent papers have considered scenarios where the learner does not even need to know that.~\cite{goldwasser2020beyond} allow the adversary to perturb test data in unrestricted ways and are still able to provide learning guarantees. The catch is that it only works in the transductive setting and only if the learner is allowed to abstain from making a prediction on some test points.~\cite{montasser2021adversarially} consider the case where the learner needs to infer the set of allowed perturbations by observing the actions of the adversary. 

Tolerance was introduced by~\cite{ashtiani2020black} in the context of certification. They provide examples where certification is not possible unless we allow some tolerance. \cite{montasser2021transductive} study transductive adversarial learning and provide a ``tolerant'' guarantee. Note that unlike our work, the main focus of that paper is on the transductive setting. Moreover, they do not specifically study tolerance with respect to metric perturbation sets. Without a metric, it is not meaningful to expand perturbation sets by a factor $(1+\gamma)$ (as we do in the our definition of tolerance). Instead, they expand their perturbation sets by applying two perturbations in succession, which is akin to setting $\gamma = 1$. In contrast, our results hold in the more common inductive setting, and capture a more realistic setting where $\gamma$ can be any (small) real number larger than zero.

\modify{Subsequent to our work,~\cite{bhattacharjee2022robust} study adversarially robust learning with tolerance for ``regular'' VC-classes and show that a simple modification of robust ERM achieves a sample complexity polynomial in both VC-dimension and doubling dimension. In a similar vein,~\cite{raman2022probabilistically} identify a more general property of hypothesis classes for which robust ERM is sufficient for adversarially robust learning with tolerance.}

Like many recent adversarially robust learning algorithms~\citep{feige2015learning, attias2018improved}, our first algorithm relies on calls to a non-robust PAC-learner.~\cite{montasser2020reducing} formalize the question of reducing adversarially robust learning to non-robust learning and study finite perturbation sets of size $k$. They show a reduction that makes $O(\log^2{k})$ calls to the non-robust learner and also prove a lower bound of $\Omega(\log{k})$. It will be interesting to see if our algorithms can be used to obtain better bounds for the tolerant setting. Our first algorithm
makes one call to the non-robust PAC-learner at training time, but needs to perform potentially expensive smoothing for making actual predictions (see Theorem~\ref{thm:tpas_guarantee}).

\modify{A related line of work studies smallest achievable robust loss for various distributions and hypothesis classes. For example, ~\cite{bubeck2021universal} show that hypothesis classes with low robust loss must be overparametrized.~\cite{yang2020closer} explore real-world datasets and provide evidence that they are separable and therefore there must exist locally Lipschitz hypotheses with low robust loss. Note that the existence of such hypotheses does not immediately imply that PAC-learning is possible.}

The techniques of randomly perturbing the training data and smoothing the output classifier has been extensively used in practice and has shown good empirical success. Augmenting the training data with some randomly perturbed samples was used for handwriting recognition as early as by ~\cite{yaeger1996effective}. More recently, ``stability training'' was introduced by~\cite{zheng2016improving} for state of the art image classifiers where training data is augmented with Gaussian perturbations. Empirical evidence was provided that the technique improved the accuracy against naturally occurring perturbations. Augmentations with non-Gaussian perturbations of a large variety were considered by~\cite{hendrycks2019augmix}.

Smoothing the output classifier using random samples around the test point is a popular technique for producing \emph{certifiably} robust classifiers. A certification, in this context, is a guarantee that given a test point $x$, all points within a certain radius of $x$ receive the same label as $x$. Several papers have provided theoretical analyses to show that smoothing produces certifiably robust classifiers~\citep{cao2017mitigating, CohenRK19, lecuyer2019certified, li2019certified, liu2018towards, SalmanLRZZBY19, levine2020robustness},~\modify{whereas others have identified cases where smoothing does not work~\cite{yang2020randomized, blum2020random}}.

However, to the best of our knowledge, a PAC-type guarantee has not been shown for any algorithm that employs training data perturbations or output classifier smoothing, and our paper provides the first such analysis. 
\section{Notations and setup}
\label{s:notations}

We denote by $X$ the input domain and by $Y=\{0,1\}$ the binary label space. We assume that $X$ is equipped with a metric $\d$. A hypothesis $h:X\to Y$ is a function that assigns a label to each point in the domain.
A hypothesis class $\H$ is a set of such hypotheses. For a sample $S = ((x_1, y_1), \ldots, (x_n, y_n))\in (X\times Y)^n$, we use the notation $S^X = \{x_1, x_2, \ldots, x_n \}$ to denote the collection of domain points $x_i$ occurring in $S$. The binary (also called 0-1) loss of $h$ on data point $(x,y)\in X\times Y$ is defined by
\[
\blo(h, x, y) = \indct{h(x) \neq y},
\]
where $\indct{.}$ is the indicator function. Let $P$ by a probability distribution over $X\times Y$. Then the \emph{expected binary loss} of $h$ with respect to $P$ is defined by
\[
\bLo{P} (h) = \Ex_{(x,y)\sim P} [\blo(h , x, y)]
\]
Similarly, the \emph{empirical binary loss} of $h$ on sample $S = ((x_1, y_1), \ldots, (x_n, y_n))\in (X\times Y)^n$ is defined as $\bLo{S}(h) = \frac{1}{n}\sum_{i=1}^n \blo(h, x_i, y_i)$. We also define the \emph{approximation error} of $\cH$ with respect to $P$ as $\bLo{P} (\cH) = \inf_{h\in \cH}\bLo{P} (h)$.

A \emph{learner} $\A$ is a function that takes in a finite sequence of labeled instances $S = ((x_1, y_1), \ldots, (x_n, y_n))$ and outputs a hypothesis $h = \A(S)$. The following definition abstracts the notion of PAC-learning~\cite{vapnikcherv71, Valiant84}.

\begin{definition}[PAC-learner]\label{def:learn}
Let $\P$ be a set of distributions over $X\times Y$ and $\H$ a hypothesis class. We say $\A$ PAC-learns $(\H, \P)$ with $m_\A: (0,1)^2\to \mathbb{N}$ samples if the following holds:
for every distribution $P\in \P$ over $X\times Y$, and every $\varepsilon,\delta \in (0,1)$, if $S$ is an \iid sample of size at least $m_\A(\varepsilon, \delta)$ from $P$, then with probability at least $1-\delta$ (over the randomness of $S$) we have
\[
\Lo{P}(\A(S)) \leq \Lo{P}(\cH) + \varepsilon.
\]
$\A$ is called an \emph{agnostic learner} if $\P$ is the set of all distributions on $X\times Y$, and a \emph{realizable learner} if $\P=\{P:\Lo{P}(\cH) = 0\}$.
\end{definition}

The smallest function $m: (0,1)^2\to \mathbb{N}$ for which there exists a learner $\A$ that satisfies the above definition with $m_{\A} = m$ is referred to as the (realizable or agnostic) \emph{sample complexity} of learning $\cH$.

The existence of sample-efficient PAC-learners for VC classes is a standard result~\cite{vapnikcherv71}. We state the results formally in Appendix~\ref{app_pac_learning}.

\subsection{Tolerant adversarial PAC-learning}

Let $\U:X\to 2^X$ be a function that maps each point in the domain to the set of its ``admissible'' perturbations. We call this function the \emph{perturbation type}. The adversarial loss of $h$ with respect to $\U$ on $(x,y)\in X\times Y$ is defined by
\[
 \rlo{\U}(h, x, y) = \max_{z\in\U(x)} \{\blo(h, z, y)\}
\]
The \emph{expected adversarial loss} with respect to $P$ is defined by $\rLo{\U}{P}(h)=\Ex_{(x,y)\sim P}\rlo{\U}(h, x, y)$. The \emph{empirical adversarial loss} of $h$ on sample $S = ((x_1, y_1), \ldots, (x_n, y_n))\in (X\times Y)^n$ is defined by $\rLo{\U}{S}(h) = \frac{1}{n}\sum_{i=1}^n \rlo{\U}(h, x_i, y_i)$. Finally, the \emph{adversarial approximation error} of $\cH$ with respect to $\U$ and $P$ is defined by $\rLo{\U}{P} (\cH) = \inf_{h\in \cH}\rLo{\U}{P} (h)$.

The following definition generalizes the setting of PAC adversarial learning to what we call the \emph{tolerant} setting, where we consider two perturbation types $\U$ and $\V$. We say $\U$ is \emph{contained in} $\V$ and and write it as $\U \prec \V$ if $\U(x)\subset\V(x)$ for all $x\in X$.

\begin{definition}[Tolerant Adversarial PAC-learner]\label{def:adv_learn_tol}
Let $\P$ be a set of distributions over $X\times Y$, $\cH$ a hypothesis class, and $\U \prec \V$ two perturbation types.
We say $\A$ \emph{tolerantly} PAC-learns $(\cH, \P, \U, \V)$ with $m_\A: (0,1)^2\to \mathbb{N}$ samples if the following holds:
for every distribution $P\in\P$ and every $\varepsilon,\delta \in (0,1)$, if $S$ is an \iid sample of size at least $m_\A(\varepsilon, \delta)$ from $P$, then with probability at least $1-\delta$ (over the randomness of $S$) we have
\[
\rLo{\U}{P}(\A(S)) \leq \rLo{\V}{P}(\cH) + \varepsilon.
\]
We say $\A$ is a tolerant PAC-learner in the \emph{agnostic setting} if $\P$ is the set of all distributions over $X\times Y$, and in the \emph{tolerantly realizable setting} if $\P=\{P:\rLo{\V}{P}(\cH) = 0\}$.
\end{definition}

In the above context, we refer to $\U$ as the \emph{actual perturbation type} and to $\V$ as the \emph{reference perturbation type}. The case where $\U(x)=\V(x)$ for all $x\in X$ corresponds to the usual adversarial learning scenario (with no tolerance). 

\subsection{Tolerant adversarial PAC-learning in metric spaces}

If $X$ is equipped with a metric $\d(.,.)$, then $\U(x)$ can be naturally defined by a ball of radius $r$ around $x$, i.e.,  $\U(x)=\B_r(x) = \{z\in X ~\mid~ \d(x,z) \leq r\}$. To simplify the notation, we sometimes use $\rlo{r}(h, x, y)$ instead of $\rlo{\B_r}(h, x, y)$ to denote the adversarial loss with respect to $\B_r$.

In the tolerant setting, we consider the perturbation sets $\U(x) = \B_r(x)$ and $\V(x) = \B_{(1+\gamma)r}(x)$, where $\gamma>0$ is called the \emph{tolerance parameter}. Note that $\U \prec \V$. We now define PAC-learning with respect to the metric space.
 
\begin{definition}[Tolerant Adversarial Learning in metric spaces]\label{def:adv_learn_metric}
Let $(X, \d)$ be a metric space, $\cH$ a hypothesis class, and $\P$ a set of distributions of $X\times Y$. We say $(\cH, \P, \d)$ is tolerantly PAC-learnable with $m:(0,1)^3\to \mathbb{N}$ samples when for every $r, \gamma>0$ there exist a PAC-learner $\A_{r, \gamma}$ for $(\cH,\P, B_r, B_{r(1+\gamma)})$ that
uses $m(\varepsilon, \delta, \gamma)$ samples.
\end{definition}

\begin{remark}
In this definition the learner receives $\gamma$ and $r$ as input but its sample complexity does not depend on $r$ (but can depend on $\gamma$). Also, as in Definition~\ref{def:adv_learn_tol}, the tolerantly realizable setting corresponds to $\P=\{P:\rLo{r(1+\gamma)}{P}(\H) = 0\}$ while in the agnostic setting $\P$ is the set of all distributions over $X\times Y$. 
\end{remark}

The doubling dimension and the doubling measure of the metric space will play important roles in our analysis. We refer the reader to Appendix~\ref{app_sec_metric} for their definitions.

We will use the following lemma in our analysis, whose proof can be found in Appendix~\ref{app_sec_metric}:
\begin{lemma}\label{lemma:simple_doubling}
For any family $\mathcal{M}$ of complete, doubling metric spaces, there exist constants $c_1, c_2 > 0$ such that for any metric space $(X, \d)\in\Mcal$ with doubling dimension $d$, there exists a measure $\mu$ such that if a ball $\B_r$ of radius $r>0$ is completely contained inside a ball $\B_{\alpha r}$ of radius $\alpha r$ (with potentially a different center) for any $\alpha > 1$, then $0<\mu(\B_{\alpha r})\leq (c_1\alpha)^{c_2 d}\mu(\B_r)$. Furthermore, if we have a constant $\alpha_0 > 1$ such that we know that $\alpha \geq\alpha_0$ then the bound can be simplified to $0 < \mu(\B_{\alpha r})\leq \alpha^{\zeta d}\mu(\B_r)$, where $\zeta$ depends on $\Mcal$ and $\alpha_0$.
\end{lemma}

Later, we will set $\alpha = 1+1/\gamma$ where $\gamma$ is the tolerance parameter. Since we are mostly interested in small values of $\gamma$, suppose we decide on some loose upper bound $\Gamma\gg\gamma$. This corresponds to saying that there exists some $\alpha_0 > 1$ such that $\alpha \geq \alpha_0$. 

It is worth noting that in the special case of Euclidean metric spaces, we can set both $c_1$ and $c_2$ to be 1. In the rest of the paper, we will assume we have a loose upper bound $\Gamma\gg\gamma$ and use the simpler bound from  Lemma~\ref{lemma:simple_doubling} extensively.

Given a metric space $(X, d)$ and a measure $\mu$ defined over it, for any subset $Z\subseteq X$ for which $\mu(Z)$ is non-zero and finite, $\mu$ induces a \emph{probability} measure $P_Z^\mu$ over $Z$ as follows. For any set $Z'\subseteq Z$ in the $\sigma$-algebra over $Z$, we define $P_Z^\mu(Z') = \mu(Z')/\mu(Z)$. With a slight abuse of notation, we write $z\sim Z$ to mean $z\sim P_Z^\mu$ whenever we know $\mu$ from the context. 

Our learners rely on being able to sample from $P_Z^\mu$. Thus we define the following oracle, which can be implemented efficiently for $\ell_p$ spaces.
\begin{definition}[Sampling Oracle]
Given a metric space $(X, \d)$ equipped with a doubling measure $\mu$, a \emph{sampling oracle} is an algorithm that when queried with a $Z\subseteq X$ such that $\mu(Z)$ is finite, returns a sample drawn from $P_Z^\mu$. We will use the notation $z\sim Z$ for queries to this oracle.
\end{definition}

\section{The perturb-and-smooth approach for tolerant adversarial learning}
\label{sec:erm}

In this section we focus on tolerant adversarial PAC-learning in metric spaces (Definition~\ref{def:adv_learn_metric}), and show that VC classes are tolerantly PAC-learnable in the tolerantly realizable setting. Interestingly, we prove this result using an approach that resembles the ``perturb-and-smooth'' paradigm which is used in practice (for example by~\cite{CohenRK19}). The overall idea is to ``perturb'' each training point $x$, train a classifier on the ``perturbed'' points, and ``smooth out'' the final hypothesis using a certain majority rule. 

We employ three perturbation types: $\U$ and $\V$ play the role of the \emph{actual} and the \emph{reference} perturbation type respectively.
Additionally, we consider a perturbation type $\W:X\to 2^X$, which is used for smoothing.
We assume $\U \prec \V$ and $\W \prec \V$.  
For this section, we will use metric balls for the three types. Specifically, if $\U$ consists of balls of radius $r$ for some $r>0$, then $\W$ will consists of balls of radius $\gamma r$ and $\V$ will consist of balls of radius $(1+\gamma)r$.

\begin{definition}[Smoothed classifier]\label{def:smooth}
For a hypothesis $h:X\to \{0,1\}$, and perturbation type $\W: X \to 2^X$, we let $\bar{h}_{\W}$ denote the classifier resulting from replacing the label $h(x)$ with the average label over $\W(x)$, that is
\[
\bar{h}_{\W}(x) = \indct{\Ex_{x'\sim \W(x)}h(x')\geq 1/2}
\]
For metric perturbation types, where $\W$ is a ball of some radius $r$, we also use the notation $\bar{h}_{r}$
and when the type $\W$ is clear from context, we may omit the subscript altogether and simply write $\bar{h}$ for the smoothed classifier.
\end{definition}

\paragraph{The tolerant perturb-and-smooth algorithm} 
We propose the following learning algorithm, \tpas, for tolerant learning in metric spaces. 
Let the perturbation radius be $r>0$ for the actual type $\U = \B_r$, and let $S = ((x_1, y_1), \ldots, (x_m, y_m))$ be the training sample. For each $x_i\in S^X$, the learner samples a point $x'_i\sim\B_{r\cdot(1+\gamma)}(x_i)$ (using the sampling oracle) from the expanded reference perturbation set $\V(x_i) = \B_{(1+\gamma)r}(x_i)$. 
Let $S' = ((x'_1, y_1), \ldots, (x'_m, y_m))$. \tpas then invokes a (standard, non-robust) PAC-learner $\A_\H$ for the hypothesis class $\Hcal$ on the perturbed data $S'$. 
We let $\hat{h} = \A_{\H}(S')$ denote the output of this PAC-learner.
Finally, \tpas outputs the $\W$-smoothed version of $\bar{h}_{\gamma r}$ for $\W= \B_{\gamma r}$.
That is, $\bar{h}_{\gamma r}(x)$ is simply the majority label in a ball of radius $\gamma r$ around $x$ with respect to the distribution defined by $\mu$, see also Definition \ref{def:smooth}. 
We will prove below that this $\bar{h}_{\gamma r}$ has a small $\U$-adversarial loss. 
Algorithm \ref{alg:tpas} below summarizes our learning procedure.

\begin{algorithm}
\caption{Tolerant Perturb and Smooth (\tpas)}\label{alg:tpas}
\begin{algorithmic}
\STATE {\bf Input:} Radius $r$, tolerance parameter $\gamma$, 
data $S =  ((x_1, y_1), \ldots, (x_m, y_m))$, accesss to sampling oracle $\Ocal$ for $\mu$ and PAC-learner $\A_{\H}$.
\STATE Initialize $S' = \emptyset$
\FOR{$i = 1$ to $m$}
\STATE Sample $x'_i \sim \B_{(1+\gamma) r}(x_i)$
\STATE Add $(x'_i, y_i)$ to $S'$
\ENDFOR
\STATE Set $\hat{h} = \A_{\H}(S')$
\STATE {\bf Output:} $\bar{h}_{\gamma r}$ defined by 
\STATE \qquad\qquad $\bar{h}_{\gamma r}(x) = \indct{\Ex_{x'\sim\B_{\gamma r}(x)} \hat{h}(x') \geq 1/2}$
\end{algorithmic}
\end{algorithm}

The following is the main result of this section.

\begin{theorem}\label{thm:tpas_guarantee}
Let $(X,\d)$ be an any metric space with doubling dimension $d$ and doubling measure $\mu$. Let $\cO$ be a sampling oracle for $\mu$. Let $\cH$ be a hypothesis class and $\P$ a set of distributions over $X\times Y$. Assume $\A_\cH$ PAC-learns $\cH$ with $m_\cH(\varepsilon, \delta)$ samples in the realizable setting. Then there exists a learner $\A$, namely \tpas, that 
\begin{itemize}
    \item Tolerantly PAC-learns $(\cH, \P, \d)$ in the tolerantly realizable setting with sample complexity bounded by
    $m(\varepsilon, \delta, \gamma)=O\left(
    m_\cH(\varepsilon, \delta)\cdot (1+1/\gamma)^{\zeta d}\right) = O\left(\frac{\vc(\Hcal) + \log{1/\delta}}{\varepsilon}\cdot (1+1/\gamma)^{\zeta d}\right)$, where $\gamma$ is the tolerance parameter and $d$ is the doubling dimension.
    \item Makes only one query to $\A_\cH$
    \item Makes $m(\varepsilon, \delta, \gamma)$ queries to sampling oracle $\cO$
\end{itemize}
\end{theorem}

The proof of this theorem uses the following key technical lemma (its proof can be found in  Appendix~\ref{app_lemma}):
\begin{lemma}\label{lem:majorities}
Let $r>0$ be a perturbation radius, $\gamma>0$ a tolerance parameter, and $g:X\to Y$ a classifier. 
For $x\in X$ and $y\in Y = \{0,1\}$, we define 
$$\Sigma_{g, y}(x) = \Ex_{z\sim\B_{r(1+\gamma)}(x)}\indct{g(z)\neq y}  \quad\text{and}\quad 
\sigma_{g,y}(x) = \Ex_{z\sim\B_{r\gamma}(x)}\indct{g(z)\neq y}.$$  
Then $\Sigma_{g,y}(x)\leq\frac{1}{3}\cdot\left(\frac{1+\gamma}{\gamma}\right)^{-\zeta d}$ implies that $\sigma_{g,y}(z)\leq 1/3$ for all $z\in\B_r(x)$.
\end{lemma}

\begin{proof}[Proof of Theorem \ref{thm:tpas_guarantee}]
Consider some $\epsilon_0 >0$ and $0<\delta < 1$ to be given (we will pick a suitable value of $\epsilon_0$ later), and assume the PAC-learner $\A_\H$ was invoked on the perturbed sample $S'$ of size at least $m_A(\epsilon_0,\delta)$. According to definition \ref{def:learn}, this implies that
with probability $1-\delta$, the output $\hat{h} = \A_\H(S)$ has (binary) loss at most $\epsilon_0$ with respect to the data-generating distribution. Note that the relevant distribution here is the two-stage process of the original data generating distribution $P$ and the perturbation sampling according to $\V = \B_{(1+\gamma)r}$. Since $P$ is $\V$-robustly realizable, the two-stage process yields a realizable distribution with respect to the standard $0/1$-loss. Thus, we have
\[
\Ex_{(x, y)\sim P}\Ex_{z\sim\B_{r(1+\gamma)}(x)}\indct{\hat{h}(z)\neq y} \leq\epsilon_0.
\]
With Lemma \ref{lem:majorities}, this becomes
$\Ex_{(x, y)\sim P} \Sigma_{\hat{h}, y}(x) \leq\epsilon_0$.
For $\lambda >0$, Markov's inequality then yields :
\begin{align}
\Ex_{(x, y)\sim P}\indct{\Sigma_{\hat{h}, y}(x)\leq\lambda} >1-\epsilon_0/\lambda\label{eqn:markov}
\end{align}
Thus setting $\lambda = \frac{1}{3}\cdot\left(\frac{1+\gamma}{\gamma}\right)^{-\zeta d}$ and plugging in the result of the Lemma \ref{lem:majorities} to equation (\ref{eqn:markov}), we get 
$$\Ex_{(x,y)\sim P}\indct{\forall z \in\B_r(x), \sigma_{\hat{h}, y}(z)\leq 1/3}>1-\epsilon_0/\lambda.$$ 
Since $\sigma_{\hat{h}, y}(z)\leq 1/3$ implies that 
$\indct{\Ex_{z'\sim\B_{\gamma r(z)}}\hat{h}(z')\geq 1/2}=y,$
using the definition of the smoothed classifier $\bar{h}_{\gamma r}$ we get
\begin{align}
    & \Ex_{(x,y)\sim P}\indct{\exists z\in\B_r(x), \bar{h}_{\gamma r}(z)\neq y}\leq\epsilon_0/\lambda\nonumber,\\
\end{align}
which implies $\rLo{r}{P}(\bar{h}_{\gamma r}) \leq \epsilon_0/\lambda$.
Thus, for the robust learning problem, if we are given a desired accuracy $\varepsilon$ and we want $\rLo{r}{P}(\bar{h}_{\gamma r})\leq\varepsilon$,  we can pick $\epsilon_0 = \lambda\varepsilon$.
Putting it all together, we get sample complexity $m \leq O(\frac{\vc(\Hcal) + \log{1/\delta}}{\epsilon_0})$ where $\epsilon_0 = \lambda\varepsilon$, and $\lambda = \frac{1}{3}\cdot\left(\frac{1+\gamma}{\gamma}\right)^{-\zeta d}$. Therefore, $m\leq O\left(\frac{\vc(\Hcal) + \log{1/\delta}}{\varepsilon}\cdot (1+1/\gamma)^{\zeta d}\right)$.
\end{proof}

{\bf Computational complexity of the learner.}
Assuming we have access to $\cO$ and an efficient algorithm for non-robust PAC-learning in the realizable setting, 
 we can compute $\hat{h}$ efficiently. Therefore, the learning can be done efficiently in this case. However, at the prediction time, we need to compute $\bar{h}(x)$ on new test points which requires us to compute an expectation. We can instead \emph{estimate} the expectations using random samples from the sampling oracle. For a single test point $x$, if the number of samples we draw is $\Omega(\log{1/\delta})$ then with probability at least $1-\delta$ we get the same result as that of the optimal $\bar{h}(x)$. Using more samples we can boost this probability to guarantee a similar output to that of $\bar{h}$ on a larger set of test points.
 
 {\bf The traditional non-tolerant framework does not justify the use of perturb-and-smooth-type approaches.}
 The introduction of the tolerance in the adversarial learning framework is crucial for being able to prove guarantees for perturb-and-smooth-type algorithms. To see why, consider a simple case where the domain is the real line, the perturbation set is an open Euclidean ball of radius 1, and the hypothesis class is the set of all thresholds. Assume that the underlying distribution is supported only on two points: $\Pr(x=-1,y=1)=\Pr(x=1, y=0)=0.5$. This distribution is robustly realizable, but the threshold should be set exactly to $x=0$ to get a small error. However, the perturb-and-smooth method will fail because the only way the PAC-learner $\Acal_\Hcal$ sets the threshold to $x=0$ is if it receives a (perturbed) sample exactly at $x=0$, whose probability is 0. 
 
\section{Improved tolerant learning guarantees through sample compression}
\label{sec:agn}

The perturb-and-smooth approach discussed in the previous section offers a general method for tolerant robust learning. However, one shortcoming of this approach is the exponential dependence of its sample complexity with respect to the doubling dimension of the metric space. Furthermore, the tolerant robust guarantee relied on the data generating distribution being tolerantly realizable. 
In this section, we propose another approach that addresses both of these issues. The idea is to adopt the perturb-and-smooth approach within a sample compression argument.
We introduce the notion of a $(\U,\V)$-tolerant sample compression scheme and present a learning bound based on such a compression scheme, starting with the realizable case. We then show that this implies learnability in the agnostic case as well. Remarkably, this tolerant compression based analysis will yield bounds on the sample complexity that avoid the exponential dependence on the doubling dimension.

For a compact representation, we will use the general notation $\U, \V, $ and $\W$ for the three perturbation types (actual, reference and smoothing type) in this section and will assume that they satisfy the Property \ref{assmt:UVWsmoothing} below for some parameter $\beta >0$.
Lemma \ref{lem:majorities} implies that, in the metric setting, for any radius $r$ and tolerance parameter $\gamma$ the perturbation types $\U = \B_r$, $\V = \B_{(1+\gamma)r}$, and $\W = \B_{\gamma r}$ have this property for $\beta = \frac{1}{3}\left(\frac{1+\gamma}{\gamma} \right)^{-\zeta d}$.

\begin{property}\label{assmt:UVWsmoothing}
For a fixed $0 < \beta <1/2$, we assume that the perturbation types $\V, \U$ and $\W$ are so that for any classifier $h$ and any $x\in X$, any $y\in\{0,1\}$ if 
\[
\Ex_{z\sim \V(x)}[h(z) = y] \geq 1-\beta
\]
then $\W$-smoothed class classifier $\bar{h}_{\W}$ satisfies $\bar{h}_{\W}(z) = y$ for all $z\in\U(x)$.
\end{property}
A compression scheme of size $k$ is a pair of functions $(\kappa, \rho)$, where the compression function $\kappa: \bigcup_{i=1}^{\infty}(X\times Y)^i \to \bigcup_{i=1}^{k}(X\times Y)^i$ maps samples $S = ((x_1, y_1), (x_2, y_2), \ldots, (x_m, y_m))$ of arbitrary size to sub-samples of $S$ of size at most $k$, and $\rho:\bigcup_{i=1}^{k}(X\times Y)^i \to Y^X$ is a decompression function that maps samples to classifiers. The pair $(\kappa, \rho)$ is a sample compression scheme for loss $\lo$ and class $\H$, if for any samples $S$ realizable by $\H$, we recover the correct labels for all $(x,y)\in S$, that is, $\Lo{S}(H)=0$ implies that $\Lo{S}(\kappa \circ \rho(S))=0$.

For tolerant learning, we  introduce the following generalization of compression schemes:
\begin{definition}[Tolerant sample compression scheme]
A sample compression scheme $(\kappa, \rho)$ is a \emph{$\U, \V$-tolerant sample compression scheme} for class $\H$, if for any samples $S$ that are $\rlo{\V}$ realizable by $\H$, that is $\rLo{\V}{S}(\H)=0$, we have $\rLo{\U}{S}(\kappa \circ \rho(S))=0$.
\end{definition}

The next lemma establishes that the existence of a sufficiently small tolerant compression scheme for the class $\H$  yields bounds on the sample complexity of tolerantly learning $\H$. The proof of the lemma is based on a modification of a standard compression based generalization bound. Appendix Section \ref{app_sec_compression} provides more details.

\begin{lemma}\label{lem:tolerant_compression_generalization}
Let $\H$ be a hypothesis class and $\U$ and $\V$ be perturbation types with $\U$ included in $\V$. If the class $\H$ admits a $(\U, \V)$-tolerant compression scheme of size bounded by $k\ln(m)$ for sample of size $m$, then the class is $(\U,\V)$-tolerantly learnable in the realizable case with sample complexity bounded by $m(\varepsilon, \delta) = \tilde{O}\left(\frac{k + \ln(1/\delta)}{\varepsilon}\right)$.
\end{lemma}

We next establish a bound on the tolerant compression size for general VC-classes, which will then immediately yield the improved sample complexity bounds for tolerant learning in the realizable case. The proof is sketched here; its full version has been moved to the Appendix.

\begin{lemma}\label{lem:tolerant_compression_bound}
Let $\H\subseteq Y^X$ be some hypothesis class with finite VC-dimension $\vc(\H) <\infty$, and let $\U, \V, \W$ satisfy the conditions in Property \ref{assmt:UVWsmoothing} for some $\beta >0$. Then there exists a $(\U,\V)$-tolerant sample compression scheme for $\H$ of size $\tilde{O}\left(\vc(\H)\ln(\frac{m}{\beta})\right)$.
\end{lemma}

\begin{proof}[Proof Sketch]
We will employ a boosting-based approach to establish the claimed compression sizes. 
Let $S = ((x_1, y_1), (x_2, y_2), \ldots, (x_m, y_m))$ be a data-set that is $\rlo{\V}$-realizable with respect to $\H$. We let $S_{\V}$ denote an ``inflated data-set'' that contains all domain points in the $\V$-perturbation sets of the $x_i\in S^X$, that is
$S_{\V}^X := \bigcup_{i=1}^{m} \V(x_i)$.
Every point $z\in S_{\V}^X$ is assigned the label $y = y_i$ of the minimally-indexed $(x_i, y_i)\in S$ with $z\in \V(x_i)$, and we set $S_\V$ to be the resulting collection of labeled data-points. 

We then use the boost-by-majority method to encode a classifier $g$ that (roughly speaking) has error bounded by $\beta/m$ over (a suitable measure over) $S_\V$.  This boosting method outputs a $T$-majority vote $g(x) = \indct{\Sigma_{i=1}^T h_i(x)} \geq 1/2$ over weak learners $h_i$, which in our case will be hypotheses from $\H$. We prove that this error can be achieved with $T = 18\ln(\frac{2m}{\beta})$ rounds of boosting. We prove that each weak learner that is used in the boosting procedure can be encoded with $n = \tilde{O}(\vc(\H))$ many sample points from $S$. The resulting compression size is thus $n\cdot T = \tilde{O}\left(\vc(\H)\ln(\frac{m}{\beta})\right)$.

Finally, the error bound $\beta/m$ of $g$ over $S_\V$ implies that the error in each perturbation set $\V(x_i)$ of a sample point $(x_i, y_i)\in S$ is at most $\beta$. Property \ref{assmt:UVWsmoothing} then implies $\rLo{\U}{S} (\bar{g}_{\W}) = 0$ for the $\W$-smoothed classifier $\bar{g}_{\W}$, establishing the $(\U,\V)$-tolerant correctness of the compression scheme.
\end{proof}

This yields the following result
\begin{theorem}
Let $\H$ be a hypothesis class of finite VC-dimension and $\V, \U, \W$ be three perturbation types (actual, reference and smoothing) satisfying Property  \ref{assmt:UVWsmoothing} for some $\beta>0$. Then the sample complexity (omitting log-factors)  of $(\U,\V)$-tolerantly learning $\H$ is bounded by 
\[
m(\varepsilon, \delta) = \tilde{O}\left(\frac{\vc(\H)\ln({1}/{\beta}) + \ln({1}/{\delta})}{\varepsilon}\right)
\]
in the realizable case, and in the agnostic case by 
\[
m(\varepsilon, \delta) = \tilde{O}\left(\frac{\vc(\H)\ln({1}/{\beta}) + \ln({1}/{\delta})}{\varepsilon^2}\right)
\]
\end{theorem}

\begin{proof}
The bound for the realizable case follows immediately from Lemma \ref{lem:tolerant_compression_bound} and the subsequent discussion (in the Appendix). For the agnostic case, we employ a reduction from agnostic robust learnabilty to realizable robust learnability~\citep{MontasserHS19, moran2016sample}. The reduction is analogous to the one presented in Appendix C of \cite{MontasserHS19} for usual (non-tolerant) robust learnablity with some minor modifications. Namely, for a sample $S$, we choose the largest subsample $S'$ that is $\rlo{\V}$-realizable (this will result in competitiveness with a $\rlo{\V}$-optimal classifier), and we will use the boosting procedure described there for the $\rlo{\U}$ loss. For the sample sizes employed for the weak learners in that procedure, we can use the sample complexity for $\varepsilon = \delta = 1/3$ of an optimal $(\U,\V)$-tolerant learner in the realizable case (note that each learning problem during the boosting procedure is a realizable $(\U,\V)$-tolerant learning task). These modifications 
result in the stated sample complexity for agnostic tolerant learnability.
\end{proof}

In particular, for the doubling measure scenario (as considered in the previous section), we obtain
\begin{corollary}\label{Thm:compress_metric}
For metric tolerant learning with tolerance parameter $\gamma$ in doubling dimension $d$ the sample complexity of adversarially robust learning with tolerance in the realizable case is bounded by $m(\varepsilon, \delta) = \tilde{O}\left(\frac{\vc(\H)\zeta d \ln(1 + 1/{\gamma}) + \ln({1}/{\delta})}{\varepsilon}\right)$ and in the agnostic case by
$m(\varepsilon, \delta) = \tilde{O}\left(\frac{\vc(\H)\zeta d \ln(1 + 1/{\gamma}) + \ln({1}/{\delta})}{\varepsilon^2}\right)$.
\end{corollary}

\paragraph{Discussion of linear dependence on $\vc(\H)$}
Earlier, general compression based sample complexity bounds for robust learning with arbitrary perturbation sets exhibit a dependence on the dual VC-dimension of the hypothesis class and therefore potentially an exponential dependence on $\vc(\H)$ \citep{MontasserHS19}.
In our setting, we show that it is possible to avoid the dependence on dual-VC by exploiting both the metric structure of the domain set and the tolerant framework.
In the full proof of Lemma \ref{lem:tolerant_compression_bound}, we show that \emph{if we can encode a classifier with small error} (exponentially small with respect to the doubling dimension for the metric case) on the perturbed distribution \emph{w.r.t. larger perturbation sets}, then we can \emph{use smoothing to get a classifier that correctly classifies every point in the inner inflated sets}. 
And, as for TPaS, the tolerant perspective is crucial to exploit a smoothing step in the compression approach (through the guarantee from Property \ref{assmt:UVWsmoothing} or Lemma \ref{lem:majorities}).

More specifically, we define a tolerant compression scheme (Definition 12) that naturally extends the classic definition of compression to the tolerant framework. The compression scheme we establish in the proof of Lemma \ref{lem:tolerant_compression_bound} then borrows ideas from our perturb-and-smooth algorithm. 
Within the compression argument, we define the perturbed distribution over the sample that we want to compress with respect to the larger perturbation sets. We then use boosting to build a classifier with very small error with respect to this distribution. The nice property of boosting is that its error decreases exponentially with the number of iterations. As a result, we also get linear dependence on the doubling dimension. This classifier can be encoded using $\tilde{O}\left(T\vc(\H) \right)$ samples ($T$ rounds of boosting, and each weak classifier can be encoded using $\tilde{O}(\vc(\H)$ samples, since we can here use simple $\varepsilon$-approximations rather than invoking VC-theory in the dual space). Our decoder receives the description of these weak classifiers, combines them, and performs a final smoothing step. The smoothing step translates the exponentially small error with respect to the perturbed distribution to zero error with respect to the (inner) inflated set, thereby satisfying the requirement of a tolerant compression scheme.

\bibliography{refs}


\appendix
\section{Standard results from VC theory}\label{app_pac_learning}
Let $X$ be a domain. For hypothesis $h$ and $B\subseteq X$ let $h(B)=\left(h(b)\right)_{b\in B}$.

\begin{definition}[VC-dimension]\label{def:vc}
 We say $\cH$ shatters $B\subseteq X$ if 
 $|\{h(B):h\in \cH \}|=2^{|B|}$. The \emph{VC-dimension} of $\H$, denoted by $\vc(\H)$, is defined to be the supremum of the size of the sets that are shattered by $\H$. 
\end{definition}

\begin{theorem}[Existence of Realizable PAC-learners~\cite{hanneke2016optimal, simon2015almost, blumer1989learnability}]\label{thm:pac_realizable}
Let $\cH$ be a hypothesis class with bounded VC-dimension. Then $\cH$ is PAC-learnable in the realizable setting using  $O\left(\frac{\vc(\H)+\log(1/\delta)}{\varepsilon}\right)$ samples.
\end{theorem}

\begin{theorem}[Existence of Agnostic PAC-learners~\cite{haussler1992decision}]\label{thm:pac_agnostic}
Let $\cH$ be a hypothesis class with bounded VC-dimension. Then $\cH$ is PAC-learnable in the agnostic setting using $O\left(\frac{\vc(\H)+\log(1/\delta)}{\varepsilon^2}\right)$ samples.
\end{theorem}

\section{Metric spaces}\label{app_sec_metric}
\begin{definition}
A metric space $(X, \d)$ is called a \emph{doubling metric} if there exists a constant $M$ such that every ball of radius $r$ in it can be covered by at most $M$ balls of radius $r/2$. The quantity $\log_2{M}$ is called the \emph{doubling dimension}.
\end{definition}
\begin{definition}
For a metric space $(X,\d)$, a measure $\mu$ defined on $X$ is called a \emph{doubling measure} if there exists a constant $C$, such that for all $x\in X$ and $r\in\reals^+$, we have that $0<\mu(\B_{2r}(x))\leq C\cdot\mu(\B_{r}(x))<\infty$. In this case, $\mu$ is called $C$-doubling.
\end{definition}
It can be shown~\citep{luukkainen1998every} that every 
\emph{complete} metric with doubling dimension $d$ has a $C$-doubling measure $\mu$ for some $C\leq 2^{cd}$ where $c$ is a universal constant. For example, Euclidean spaces with an $\ell_p$ distance metric are complete and the Lesbesgue measure is a doubling measure.

The following lemmas follow straightforwardly from the definitions doubling metric and measures.
\begin{lemma}\label{lemma:initial_doubling}
Let $(X, \d)$ be a doubling metric 
equipped with a $C$-doubling measure $\mu$. Then for all $x\in X$, $r>0$, and $\alpha > 1$, we have that $\mu(\B_{\alpha r}(x)) \leq C^{\lceil\log_2\alpha\rceil}\cdot\mu(\B_r(x))$
\end{lemma}
\begin{proof}
Since $\mu$ is a measure, if $B, B'\subseteq X$ such that $B\subseteq B'$, then $\mu(B)\leq\mu(B')$.
Let $R = 2^{\lceil\log_2\alpha\rceil}$. It's clear that $R\geq\alpha$. Therefore $\B_{\alpha r}(x)\subseteq\B_R(x)$. Expanding $\B_r(x)$ by a factor of two $\lceil\log_2\alpha\rceil$ times, we get $\B_R(x)$, which means $\mu(\B_R(x))\leq C^{\lceil\log_2\alpha\rceil}\cdot\mu(\B_r(x))$. But since $\B_{\alpha r}(x)\subseteq\B_R(x)$, we get the desired result.
\end{proof}
\begin{lemma}\label{lemma:doubling}
Let $(X, \d)$ be a doubling metric equipped with a $C$-doubling measure  $\mu$. Let $x, x'\in X$, $r>0$, and $\alpha>1$ be such that $\B_r(x')\subseteq\B_{\alpha r}(x)$. Then $\mu(\B_{\alpha r}(x))\leq C^{\lceil\log_2(2\alpha)\rceil}\cdot\mu(\B_r(x'))$.
\end{lemma}
\begin{proof}
By Lemma \ref{lemma:initial_doubling}, all we need to show is that $B_{\alpha r}(x)\subseteq\B_{2\alpha r}(x')$. Indeed, let $y\in\B_{\alpha r}(x)$ be any point. Then, from triangle inequality, we have that 
\begin{align}
d(x', y) & \leq d(x, x')+d(x, y)\nonumber\\    
         & \leq d(x, x') + \alpha r\nonumber
\end{align}
Moreover, since $x'\in\B_{\alpha r}(x)$, we have that $d(x, x') \leq \alpha r$. Substituting into the equation above, we get $d(x', y)\leq 2\alpha r$, which means $y\in\B_{2\alpha r}(x')$.
\end{proof}
Finally, we also get:
\begin{lemma}\label{lemma:simple_doubling}
For any family $\mathcal{M}$ of complete, doubling metric spaces, there exist constants $c_1, c_2 > 0$ such that for any metric space $(X, \d)\in\Mcal$ with doubling dimension $d$, there exists a measure $\mu$ such that if a ball $\B_r$ of radius $r>0$ is completely contained inside a ball $\B_{\alpha r}$ of radius $\alpha r$ (with potentially a different center) for any $\alpha > 1$, then $0<\mu(\B_{\alpha r})\leq (c_1\alpha)^{c_2 d}\mu(\B_r)$. 
\end{lemma}
\begin{proof}
We prove this when $\Mcal$ is the set of all complete, doubling metric spaces employing Lemmas \ref{lemma:initial_doubling}
and \ref{lemma:doubling}, that can be found in Appendix, part \ref{app_sec_metric}.
We have that $C^{\lceil\log_2 (2\alpha)\rceil}\leq (2\alpha)^{2\log_2 C}$. 
Since $\log_2 C\leq cd$, we get $(2\alpha)^{2\log_2 C}\leq (2\alpha)^{cd}$. Thus $c_1 = 2$ and $c_2=c$.
\end{proof}
\begin{corollary}
\label{corr:doubling}
Suppose we have a constant $\alpha_0 > 1$ such that we know that $\alpha \geq\alpha_0$. Then the bound in Lemma~\ref{lemma:simple_doubling} can be further simplified to $0 < \mu(\B_{\alpha r})\leq \alpha^{\zeta d}\mu(\B_r)$, where $\zeta$ depends on $\Mcal$ and $\alpha_0$. Furthermore, if $c_1 = 1$ then we can set $\alpha_0 = 1$.
\end{corollary}
\begin{proof}
$(c_1\alpha)^{c_2d} = \alpha^{c_2d(1+\log_\alpha c_1)}\leq\alpha^{c_2d(1+\log_{\alpha_0} c_1)} = \alpha^{\zeta d}$ for $\zeta = c_2(1+\log_{\alpha_0} c_1)$. If $c_1=1$, then $\zeta=c_2$ for all $\alpha$.
\end{proof}

\section{Proof of Lemma~\ref{lem:majorities}}
\label{app_lemma}
Let
$X_{\text{err}} = \{z\in\B_{r(1+\gamma)}(x)\mid\ g(z)\neq y\}$. Then, we have that 
$\Sigma_{g, y}(x) = \Ex_{z\sim\B_{r(1+\gamma)}(x)}\indct{g(z)\neq y} = \frac{\mu(X_{\text{err}})}{\mu(\B_{r(1+\gamma)}(x))}$.
Further, for all $ z\in\B_r(x)$, we have  $\Ex_{z'\sim\B_{r\gamma}(z)}\indct{g(z')\neq y} = \frac{\mu(X_{\text{err}}\cap\B_{r\gamma}(z))}{\mu(\B_{r\gamma}(z))}$.

Let $z\in \B_r(x)$. Since this implies that $\B_{r\gamma}(z)\subseteq\B_{r(1+\gamma)}(x)$, the worst case happens when $X_{\text{err}}\subseteq\B_{r\gamma}(z)$. 
Therefore, 
\begin{align}
    \sigma_{g, y}(x) & = \Ex_{z'\sim\B_{r\gamma}(z)}\indct{g(z')\neq y} \\ 
    & = \frac{\mu(X_{\text{err}}\cap\B_{r\gamma}(z))}{\mu(\B_{r\gamma}(z))}\nonumber\\
    & \leq \frac{\mu(X_{\text{err}})}{\mu(\B_{r\gamma}(z))}\nonumber\\
    & \leq\frac{\Sigma_{g, y}(x)\cdot\mu(\B_{r(1+\gamma)}(x))}{\mu(\B_{r\gamma}(z))}\nonumber\\
    & \leq\Sigma_{g, y}(x)\cdot \left(\frac{1+\gamma}{\gamma}\right)^{\zeta d}\nonumber,
\end{align}
where the last inequality is implied by Lemma \ref{lemma:simple_doubling}.
Thus, $\Sigma(x)\leq\frac{1}{3}\cdot\left(\frac{1+\gamma}{\gamma}\right)^{-\zeta d}$ implies that $\sigma(z)\leq 1/3$ as claimed.

\section{Compression based bounds}\label{app_sec_compression}

\subsection{Proof of Lemma \ref{lem:tolerant_compression_generalization}}

To prove the generalization bound for tolerant learning, we employ the following lemma that establishes generalization for  compression schemes for adversarial losses:
\begin{lemma}[Lemma 11, \citep{MontasserHS19}\label{lem:compression_yields_learning}]
For any $k\in \naturals$ and fixed function $\rho: \bigcup_{i=1}^{k}(X\times Y)^i \to Y^X$, for any distribution $P$ over $X\times Y$ and any $m\in\naturals$, with probability at least $(1-\delta)$ over an \iid sample $S = ((x_1, y_1), (x_2, y_2), \ldots, (x_m, y_m))$: if there exist indices $i_1, i_2, \ldots, i_k$ such that
\[
\rLo{\U}{S}(\rho((x_{i_1}, y_{i_1}), (x_{i_2}, y_{i_2}), \ldots, (x_{i_k}, y_{i_k}))) = 0
\]
then the robust loss of the decompression with respect to $P$ is bounded by
\begin{align*}
& \rLo{\U}{P}(\rho((x_{i_1}, y_{i_1}), (x_{i_2}, y_{i_2}), \ldots, (x_{i_k}, y_{i_k}))) \\
& \qquad\qquad  ~\leq~  \frac{1}{m-k}(k \ln(m) + \ln(1/\delta))
\end{align*}

\end{lemma}

The above lemma implies that if $(\kappa, \rho)$ is a compression scheme that compresses data-sets of size $m$ to at most $k\ln(m)$ data points, for class $\H$ and robust loss $\rlo{\U}$, then the sample complexity (omitting logarithmic factors) of robustly learning $\H$ in the realizable case is bounded by
\[
m(\varepsilon, \delta) = \tilde{O}\left(\frac{k + \ln(1/\delta)}{\varepsilon}\right)
\]
For the tolerant setting, since every sample that is realizable with respect to $\rlo{\V}$ is also realizable with respect to $\rlo{\U}$, if a $(\U,\V)$-tolerant compression scheme compresses to at most $k\ln(m)$ data-points and decompresses all $\rlo{\V}$-realizable samples $S$ to functions that have $\rlo{\U}$-loss $0$ on $S$, then the lemma  implies the above bound for the $(\U,\V)$-tolerant sample complexity of learning $\H$.

\subsection{ Proof of Lemma \ref{lem:tolerant_compression_bound}}
The proof of this Lemma employs the notions of a sample being $\varepsilon$-net or and $\varepsilon$-approximation for a hypothesis class $\H$. A labeled data set $S = ((x_1, y_1), (x_2, y_2), \ldots, (x_m, y_m))$ is an $\varepsilon$-net for class $\H$ with respect to distribution $P$ over $X\times Y$ if for every hypothesis $h\in\H$ with $\bLo{P}(h) \geq \varepsilon$, there exists an index $j$ and $(x_j, y_j)\in S$ with $h(x_j) \neq y_j$. $S$ is an $\varepsilon$-approximation for class $\H$ with respect to distribution $P$ over $X\times Y$ if for every hypothesis $h\in\H$ we have $|\bLo{S}(h) - \bLo{P}(h)| \leq \varepsilon$. Standard VC-theory tells us that, for classes with bounded VC-dimension, sufficiently large samples from $P$ are $\varepsilon$-nets or $\varepsilon$-approximations with high probability \citep{HausslerW87}.
 
\begin{proof}
We will employ a boosting-based approach to establish the claimed compression sizes. 
Let $S = ((x_1, y_1), (x_2, y_2), \ldots, (x_m, y_m))$ be a data-set that is $\rlo{\V}$-realizable with respect to $\H$. We let $S_{\V}$ denote an ``inflated data-set'' that contains all domain points in the perturbation sets of the $x_i\in S^X$, that is
\[
S_{\V}^X := \bigcup_{i=1}^{m} \V(x_i)
\]
Every point $z\in S_{\V}^X$ is assigned the label $y = y_i$ of the minimally-indexed $(x_i, y_i)\in S$ with $z\in \V(x_i)$, and we set $S_\V$ to be the resulting collection of labeled data-points. (Note that since the sample $S$ is assumed to be $\rlo{\V}$-realizable, assigning it the label of some other corresponding data point in case $z\in \V(x_i)\cap \V(x_j)$ for $x_i\neq x_j$, would not induce any inconsistencies).  Now let $D$ be the probability measure over $S_{\V}^X$ defined by first sampling an index $j$ uniformly from $[j] = \{1,2,\ldots, j\}$ and then sampling a domain point $z\sim \V(x_j)$ from the $\V$-perturbation set around the $j$-th sample point in $S$. Note that this implies 
that if $D(B) \leq (\beta/m)$ for some subset $B\subseteq S_{\V}^X$, then
\begin{equation}\label{eqn_weight_distribution}
\Pr_{z\sim \V(x)}[z\in B] \leq \beta    
\end{equation}
for all $x\in S^X$.

We will now show that, by means of a compression scheme, we can encode a hypothesis $g$ with binary loss 
\begin{equation}\label{loss_goal}
\bLo{D}(g) \leq \beta/m.
\end{equation}
Property \ref{assmt:UVWsmoothing} together with Equation \ref{eqn_weight_distribution}
then implies that the resulting $\W$-smoothed function $\bar{g}$ has $\U$-robust loss $0$ on the sample $S$, $\rLo{\U}{S}(\bar{g}) = 0$. Since the smoothing is a deterministic operation once $g$ is fixed, this implies the existence of a $(\U, \V)$-tolerant compression scheme.

Standard VC-theory tells us that, for a class $G$ of bounded VC-dimension, for any distribution over $X\times Y$, and any $\varepsilon,\delta >0$,  with probability at least $(1-\delta)$ an \iid sample of size $\Theta\left(\frac{\vc(G)+ \ln(1/\delta)}{\varepsilon^2} \right)$ is an $\varepsilon$-approximation for the class $G$ \citep{HausslerW87}. This implies in particular, that there exists a finite subset $S_{\V}^f \subset S_{\V}$ of size at most $\frac{4 m^2 C\cdot \vc(G)}{\beta^2}$ (for some constant $C$) with the property that any classifier $g\in G$ with empirical (binary) loss at most $\beta/2m$ on  $S_\V^f$ has loss $\bLo{D}(g) \leq \beta/m$ with respect to the distribution $D$. We will choose such a set  $S_\V^f$ for the class $G$ of $T$-majority votes over $\H$ for $T = 18\ln(\frac{2m}{\beta})$. That is
\begin{align*}
G = \{g\in Y^X  & \mid \exists h_1, h_2, \ldots, h_T \in H: \\
 & g(x) = \indct{\Sigma_{i=1}^T h_i(x) \geq 1/2}\}
\end{align*}
The VC-dimension of $G$ is bounded by~\citep{shalev2014understanding}
\begin{align*}
\vc(G) = \Ocal(T\cdot\vc(\H)\log(T\vc(\H))) = \Ocal(18\ln(\frac{2m}{\beta})\vc(\H)\log(18\ln(\frac{2m}{\beta})\vc(\H))).
\end{align*}

We will now show how to obtain the classifier $g$ by means of a boosting approach on the finite data-set $S_\V^f$. More specifically, we will use the boost-by-majority method. This method outputs a $T$-majority vote $g(x) = \indct{\Sigma_{i=1}^T h_i(x)} \geq 1/2$ over weak learners $h_i$, which in our case will be hypotheses from $\H$. After $T$ iterations with $\gamma$-weak learners,
the empirical loss over the sample $S_\V^f$ is bounded by $\e^{-2\gamma^2 T}$ (see Section 13.1 in \citep{schapire2013boosting}). Thus, with $\gamma =1/6$, and $T = 18\ln(\frac{2m}{\beta})$, we obtain
\[
\bLo{S_\V^f}(g) \leq \frac{\beta}{2m}
\]
which, by the choice of $S_\V^f$ implies
\[
\bLo{D}(g) \leq \beta/m
\]
which is what we needed to show according to Equation \ref{loss_goal}.

It remains to argue that the weak learners to be employed in the boosting procedure can be encoded by a small number of sample points from the original sample $S$. For this part, we will employ a technique introduced earlier for robust compression \citep{MontasserHS19}. Recall that the set $S$ is $\V$-robustly realizable, which implies that the set $S_\V^f$ is (binary loss-) realizable by $\H$. By standard VC-theory, for every distribution $D_i$ over $S_\V^f$, there exists an $\varepsilon$-net of size $\Ocal(\vc(\H)/\varepsilon)$ \citep{HausslerW87}. Thus, for every distribution $D_i$ over $S_\V^f$  (that may occur during the boosting procedure), there exists a subsample $S_i$ of $S_\V^f$, of size at most $n = \Ocal(3\vc(\H))$ with the property that every hypothesis from $\H$ that is consistent with $S_i$ has binary loss at most $1/3$ with respect to $D_i$ (thus can serve as a weak learner for margin $\gamma = 1/6$ in the above procedure). Now for every labeled point $(x,y)\in S_i$, there is a sample point $(x_j, y_j)\in S$ in the original sample $S$ such that $x\in \V(x_j)$ and $y = y_j$. Let $S'_i$ be the collection of these corresponding original sample points. Note that any hypothesis $h\in H$ that is $\V$-robustly consistent with $S'_i$ is consistent with $S_i$. Therefore we can use the $n$ original data-points in $S'_i$ to encode the weak learner $h_i$ (for the decoding any $\V$-robust ERM hypothesis can be chosen to obtain $h_i$).

To summarize, we will compress the sample $S$ to the sequence $S'_1, S'_2, \ldots S'_T$ of $n\cdot T = \Ocal(\vc(\H)\ln(\frac{m}{\beta}))$ sample points from $S$. To decode, we obtain the function $g$ as a majority vote over the weak learner $h_i$ and proceed to obtain the $\W$-smoothed function $\bar{g}$. This function $\bar{g}$ satisfies  $\rLo{\U}{S}(\bar{g}) = 0$ and by this we have established the existence of a $\U, \V$-tolerant compression scheme of size $\Ocal(\vc(\H)\ln(\frac{m}{\beta}))$ as claimed.
\end{proof}

\end{document}